\documentclass[12pt]{article}

\usepackage{amsmath,amssymb,amsthm}
\usepackage{natbib}
\usepackage{graphicx}
\usepackage{url}
\usepackage{hyperref}
\usepackage{authblk}

\usepackage{booktabs}
\usepackage[table]{xcolor}
\usepackage{algorithm, algorithmic}
\usepackage{tcolorbox}

\usepackage{tocloft}

\allowdisplaybreaks[4]

\newcommand{\R}{\mathbb{R}}
\newcommand{\EE}{\mathbb{E}}
\newcommand{\cD}{\mathcal{D}}
\newcommand{\cO}{\mathcal{O}}
\newcommand{\hy}{\hat{y}}
\newcommand{\hv}{\hat{v}}
\newcommand{\tv}{\tilde{v}}
\newcommand{\norm}[1]{\left\|#1\right\|}
\newcommand{\dotprod}[1]{\left\langle #1\right\rangle}
\newtheorem{theorem}{Theorem}
\newtheorem{lemma}[theorem]{Lemma}
\newtheorem{corollary}[theorem]{Corollary}
\newtheorem{definition}[theorem]{Definition}
\newtheorem{remark}[theorem]{Remark}
\newtheorem{assumption}{Assumption}
\newenvironment{keywords}{\par\noindent\textbf{Keywords: }\ignorespaces}{\par}

\title{On the Convergence of Single-Loop Stochastic Bilevel Optimization with Approximate Implicit Differentiation}
\usepackage{times}

\author[1]{Yubo Zhou\thanks{Email: \href{mailto:ybzhouni@gmail.com}{\texttt{ybzhouni@gmail.com}}}}
\author[2]{Luo Luo}
\author[3]{Guang Dai}
\author[1]{Haishan Ye}

\affil[1]{Xi’an Jiaotong University}
\affil[2]{Fudan University}
\affil[3]{SGIT AI Lab}
\date{}

\begin{document}

\maketitle

\begin{abstract}%
\iffalse
  Stochastic Bilevel Optimization has emerged as a fundamental framework for meta-learning and hyperparameter optimization. Despite the practical prevalence of single-loop algorithms—which update lower and upper variables concurrently—their theoretical understanding, particularly in the stochastic regime, remains significantly underdeveloped compared to their multi-loop counterparts. Existing analyses often yield suboptimal convergence rates or obscure the critical dependence on the lower-level condition number $\kappa$, frequently burying it within generic Lipschitz constants. In this paper, we bridge this gap by providing a refined convergence analysis of the Single-loop Stochastic Approximate Implicit Differentiation (SSAID) algorithm. We prove that SSAID achieves an $\epsilon$-stationary point with an oracle complexity of $\mathcal{O}(\kappa^7 \epsilon^{-2})$. Our result is noteworthy in two aspects: (i) it matches the optimal $\mathcal{O}(\epsilon^{-2})$ rate of state-of-the-art multi-loop methods (e.g., stocBiO) while maintaining the computational efficiency of a single-loop update; and (ii) it provides the first explicit, fine-grained characterization of the $\kappa$-dependence for stochastic AID-based single-loop methods. This work demonstrates that SSAID is not merely a heuristic approach, but admits a rigorous theoretical foundation with convergence guarantees competitive with mainstream multi-loop frameworks.%
\fi
{Stochastic Bilevel Optimization has emerged as a fundamental framework for meta-learning and hyperparameter optimization. Despite the practical prevalence of single-loop algorithms, their theoretical understanding in the stochastic regime remains less developed than that of multi-loop methods. In this paper, we provide a refined convergence analysis of the Single-loop Stochastic Approximate Implicit Differentiation (SSAID) algorithm. Under the squared-gradient stationarity criterion $\|\nabla\Phi(x)\|^2\le\epsilon$, the corrected proof establishes an oracle complexity of $\mathcal{O}(\kappa^{14}\epsilon^{-2})$, equivalently an averaged stationarity rate of $\mathcal{O}(\kappa^7K^{-1/2})$. The result preserves the canonical $\mathcal{O}(\epsilon^{-2})$ dependence on the target accuracy while giving an explicit characterization of the condition-number dependence for stochastic AID-based single-loop methods.}%
\end{abstract}


\addtocontents{toc}{\protect\setcounter{tocdepth}{-1}}
\section{Introduction}

Bilevel Optimization (BLO) has become a cornerstone formulation for modern machine learning problems such as hyperparameter optimization \citep{franceschi2017forward, grazzi2020iteration, bao2021stability}, meta-learning \citep{finn2017model, shu2019meta}, and neural architecture search \citep{liu2018darts, he2020milenas}. We consider a general stochastic bilevel problem which takes the following formulation:
\begin{equation}\label{eq:bo-prob}
    \begin{aligned}
        \min_{x \in \R^{m}} \Phi(x) &=f(x, y^*(x)),
    \\
    y^*(x) &= \arg\min_{y\in \R^{n}}g(x, y).
    \end{aligned}
\end{equation}
where the upper- and lower-level functions $f$ and $g$ are both jointly continuously differentiable on $\R^m\times\R^n$. We focus on the standard setting \citep{ji2021bilevel, ji2022will} where $g(x, y):=\EE_{\zeta} [G(x, y; \zeta)]$ is strongly convex with respect to (w.r.t.) the lower-level variable $y$, which can guarantee the uniqueness of the lower solution \citep{chen2024finding}, and the upper-level objective function $f(x, y)=\EE_{\xi} [F(x, y; \xi)]$ is nonconvex w.r.t $x$.

A key challenge in BLO is the estimation of the hypergradient $\nabla \Phi(x)$, which involves the Jacobian of the best-response map $y^*(x)$. Under mild assumptions, $\nabla \Phi(x)$ has the following formulation:
\begin{align}
  \nabla \Phi(x) = \nabla_x f(x, y^*(x))+\nabla y^*(x)^{\top}  \nabla_y f(x, y^*(x)).
\end{align}
The difficulty is to compute $\nabla y^*(x)$. From implicit function theorem, it satisfies a linear equation:
\begin{align}
    \nabla_{xy}^2 g(x, y^*(x))+\nabla_{yy}^2 g(x, y^*(x))\nabla y^*(x)=0.
\end{align}
Thus, the gradient $\nabla y^*(x)$ can be formulated as $\nabla y^*(x) = -(\nabla_{yy}^2 g(x, y^*(x)))^{-1}\nabla_{xy}^2 g(x, y^*(x))$. Due to the computational cost of exact inversion, \textit{Approximate Implicit Differentiation (AID)} is commonly employed to estimate the inverse Hessian vector product (HVP).

AID and its variants constitute a primary class of methods for solving BLO. \cite{ghadimi2018approximation} provided the first non-asymptotic convergence guarantees for stochastic BLO by introducing the Bilevel Stochastic Approximation (BSA) algorithm. BSA is a multi-loop framework where the lower-level solution is iteratively estimated via multiple steps of Stochastic Gradient Descent (SGD). \cite{hong2023two} proposed the Two-Time-Scale Stochastic Approximation (TTSA) framework for BLO, which updates the lower and upper variables in a single iteration using different step sizes. However, the hypergradient computation in TTSA relies on multi-step stochastic Neumann expansions. Further advancing the state-of-the-art, \cite{ji2021bilevel} introduced stocBiO, a multi-loop stochastic bilevel optimizer that achieves accelerated convergence rates through a strategic warm-start mechanism. In a similar vein, inspired by the theory of singularly perturbed systems, \cite{arbel2021amortized} provided a rigorous analysis of multi-loop stochastic AID-based methods and proposed AmIGO, further characterizing the interplay between lower-level iterations and upper-level stability.
% \begin{table}[!t]
%     \centering
%     \caption{Comparison of oracle complexity to find an $\epsilon$-stationary point for the stochastic bilevel problem Eq. \eqref{eq:bo-prob} under the standard nonconvex-strongly-convex setting. For TTSA \citep{hong2023two}, although it utilizes two-timescale updates, the high truncation level of the Neumann series required for Hessian-inverse approximation prevents it from being a truly low-cost single-loop algorithm in practice.}
%     \label{tab:intro-compare}
%     \begin{tabular}{lll}
%         \toprule
%         \textbf{Type} & \textbf{Algorithm} & \textbf{Oracle Complexity} \\
%         \midrule
%         Multi-loop & BSA \citep{ghadimi2018approximation} & $\cO(\kappa^9 \epsilon^{-3})$ \\
%         Two-timescale& TTSA \citep{hong2023two} & $\cO(\kappa^{16} \epsilon^{-5/2})$ \\
%         Multi-loop & stocBiO \citep{ji2021bilevel} & $\cO(\kappa^9 \epsilon^{-2})$ \\
%         Multi-loop & AmIGO \citep{arbel2021amortized} & $\cO(\kappa^9 \epsilon^{-2})$ \\
%         \rowcolor[gray]{0.9}
%         Single-loop & SSAID (Algorithm \ref{alg:bo-ssaid}) & $\cO(\kappa^7 \epsilon^{-2})$ \\
%         \bottomrule
%     \end{tabular}
% \end{table}

Among these methodologies, a fundamental design choice concerns the precision with which the lower-level subproblem is solved. Single-loop algorithms (e.g., Algorithm~\ref{alg:bo-ssaid}) operate within a unified loop where auxiliary variables undergo only a single iteration per upper variable update. In contrast, multi-loop algorithms (e.g., BSA, stocBiO, AmIGO) require multiple lower iterations to ensure the auxiliary variables sufficiently approximate the optimal solution before each upper update. Generally, single-loop algorithms are more implementation-friendly and computationally efficient in practice, leading to their widespread adoption in various machine learning applications such as meta-learning \citep{finn2017model, snell2017prototypical}, hyperparameter optimization \citep{franceschi2017forward} and representation learning \citep{jiang2025beyond}. However, multi-loop algorithms often admit more straightforward theoretical analysis, as they allow for tighter control over the discrepancy between the numerical iterates and the analytical solutions. Multi-loop methods like stocBiO achieve a convergence rate of $\mathcal{O}(\epsilon^{-2})$ by employing multiple iterations to refine their estimates.

Beyond the aforementioned approaches, single-loop methods based on AID represent a prominent class of algorithms. While \cite{ji2022will} established a convergence rate of $\mathcal{O}(\kappa^6\epsilon^{-1})$ for the deterministic case, the convergence properties of the Single-loop Stochastic AID (SSAID) algorithm in the stochastic regime, which is of greater relevance to the machine learning community, remain an open question. Moreover, a critical theoretical aspect remains under-addressed: \textbf{the dependence on the condition number $\kappa$ of the lower-level problem.} Existing works typically present convergence bounds where $\kappa$ is hidden within ``problem-dependent constants" (e.g., the Lipschitz constant of $\nabla \Phi(x)$ or the variance of the estimators). 

However, ignoring $\kappa$ paints an incomplete picture. The hypergradient $\nabla \Phi(x)$ itself has a Lipschitz constant $L_{\Phi}$ that naturally scales with $\mathcal{O}(\kappa^3)$ \citep{ghadimi2018approximation}. For single-loop algorithms, the situation is further complicated by the \textit{coupling error} between the dynamic updates of $x, y$, and the AID estimator. If one were to explicitly unfold the constants in standard single-loop analyses, the implied complexity often scales poorly due to the loose handling of tracking errors and the step-size ratios required for stability.

\paragraph{Contributions.}
We address this theoretical gap by providing a fine-grained convergence analysis of SSAID algorithm. Our main contributions are:
\begin{itemize}
    \item \textbf{Explicit Characterization:} We move beyond ``hidden constants" and explicitly derive the dependence of the complexity on $\kappa$. We prove that the SSAID algorithm achieves an $\epsilon$-stationary point with a complexity of $\mathcal{O}(\kappa^{14} \epsilon^{-2})$.
    \item \textbf{Technical methodology:} Our result is achieved through a refined analysis of the coupling between the optimization error of solution for the lower-level subproblem and the approximation error of the solution for the linear system.
\end{itemize}

\section{Related Work}
Methods based on AID \citep{domke2012generic, pedregosa2016hyperparameter, ghadimi2018approximation, grazzi2020iteration, ji2021bilevel} estimate the product of the inverse hessian and a vector by solving linear systems with efficient iterative solvers. Under stochastic setting, the non-asymptotic analysis of stochastic BLO was pioneered by \cite{ghadimi2018approximation}, who analyzed multi-loop algorithm, BSA. While they established tight bounds $\mathcal{O}(\kappa^4 \epsilon^{-1})$ for the deterministic case, their stochastic analysis relies on large-batch sampling or repeated sub-problem solving. For stochastic case, \cite{ji2021bilevel} and \cite{arbel2021amortized} improve the convergence rate of multi-loop algorithms via employing warm-start strategy. More recently, single-loop algorithms have gained traction. \cite{hong2023two} proposed a two-timescale stochastic approximation (TTSA) framework. Although it utilizes two-timescale updates, the high truncation level of the Neumann series required for Hessian-inverse approximation prevents it from being a truly low-cost single-loop algorithm in practice. \cite{ji2022will} established the convergence rate for the single-loop AID method in the deterministic setting and demonstrated its inferiority compared to multi-loop approaches. However, the convergence properties of this method under the stochastic setting remain an open problem.

\paragraph{Comparison with Our Result.}
To the best of our knowledge, no prior work has explicitly established a complexity bound better than $\mathcal{O}(\kappa^{14} \epsilon^{-2})$ for fully single-loop stochastic AID algorithm. Indeed, a direct combination of the standard bounds for hypergradient smoothness ($L_\Phi \propto \kappa^3$) and the tracking error analysis in single-loop schemes typically implies a dependence significantly looser than our result. We close this gap by providing a tight and explicit bound.

\section{Background and Algorithm}\label{sec:back-alg}
In this section, we will present the explicit form of the gradient estimator for the BLO problem and introduce the SSAID algorithm analyzed in this work.

\textbf{Gradient Estimation.} We focus on problem (\ref{eq:bo-prob}) under a set of standard assumptions (see Section \ref{sec:assum} for details). The most straightforward approach to solving (\ref{eq:bo-prob}) is to conduct gradient descent on the upper-level objective $\Phi(x)$ (i.e., \textit{hypergradient}). Note that
\begin{align}\label{eq:hypergradient}
  \nabla \Phi(x) = \nabla_x f(x, y^*(x))+\nabla y^*(x)^{\top}  \nabla_y f(x, y^*(x)).
\end{align}
From implicit function theorem, $\nabla y^*(x)$ satisfies a linear equation:
\begin{align}
    \nabla_{xy}^2 g(x, y^*(x))+\nabla_{yy}^2 g(x, y^*(x))\nabla y^*(x)=0.
\end{align}
Thus, the hypergradient Eq. \eqref{eq:hypergradient} can be transformed as
\begin{align}
  \nabla \Phi(x) = \nabla_x f(x, y^*(x))-\nabla_{xy}^2 g(x, y^*(x))(\nabla_{yy}^2 g(x, y^*(x)))^{-1}  \nabla_y f(x, y^*(x)).
\end{align}
For the popular AID-based methods, the hypergradient can be formulated by
\begin{align}\label{eq:hg-aid}
  \nabla \Phi(x) = \nabla_x f(x, y^*(x))-\nabla_{xy}^2 g(x, y^*(x))v^*,
\end{align}
where $v^*$ is the solution of the linear system $\nabla_{yy}^2 g(x, y^*(x))v=\nabla_y f(x, y^*(x))$. To solve this linear system, AID-based algorithms can utilize many methods like Conjugate Gradient (CG).

In practice, given $x$, obtaining the exact $y^*(x)$ and $v^*$ required for the hypergradient in Eq. \eqref{eq:hg-aid} is often computationally prohibitive. Consequently, these terms are typically replaced by approximate solutions $\hat{y}$ and $\hat{v}$ generated via iterative procedures. As previously discussed, the choice of iterative scheme generally falls into two paradigms: multi-loop and single-loop approaches. Intuitively, while multi-loop methods may achieve superior convergence rates, they incur significantly higher computational overhead in practice. Conversely, single-loop methods have gained considerable popularity and are extensively deployed in stochastic settings for large-scale machine learning problems. From a theoretical perspective, however, their convergence rates may deteriorate due to the difficulty in controlling the approximation error between the iterates and the exact solutions. Notably, establishing the convergence rate of single-loop AID-based algorithms in the stochastic setting remains an open problem.

\textbf{SSAID Algorithm.} In light of the aforementioned analysis, we focus our attention on the single-loop stochastic AID-based algorithm (SSAID). The detailed procedure of SSAID is formally summarized in Algorithm \ref{alg:bo-ssaid}. SSAID addresses problem (\ref{eq:bo-prob}) by employing a single-loop framework with warm-start tracking, ensuring both computational efficiency and theoretical convergence.

The logic of SSAID can be decomposed into three theoretically motivated stages. 
\begin{itemize}
    \item[1.] \textbf{Warm-Start Tracking of the Lower-Level subproblem.} At each iteration $k$, SSAID performs a single gradient descent step to update the lower-level variable $\hy_k$. Form theoretical view, to estimate $\nabla \Phi(x)$ accurately, the algorithm must track $y^*(x)$ as $x_k$ evolves. Rather than solving the lower-level problem to high precision at every step, SSAID utilizes a warm start ($y_k^0 = \hy_{k-1}$). This leverages the regularity of the optimal solution path; if $y^*(x)$ is Lipschitz continuous and the step size $\beta$ is small, $\hy_{k-1}$ remains a high-quality initialization for the new lower-level objective at $x_k$, allowing a single step to maintain a controllable tracking error $\norm{\hy_k - y^*(x_k)}$.
    \item[2.] \textbf{Adjoint Variable Estimation via AID.} The algorithm introduces an auxiliary variable $v_k$ to approximate the inverse HVP term. The update for $\hv_k$ (Step 8) is essentially a single step of an iterative solver (similar to a Richardson iteration or a Neumann series approximation) aimed at solving the linear system $\nabla_{yy}^2 G(x,y) \cdot v = \nabla_y F(x,y)$. By again using the warm start for $v_k$, the algorithm ensures that $v_k$ converges to the adjoint solution, bypassing the need for multiple lower iterations to solve the linear system from scratch.
    \item[3.] \textbf{Stochastic Hypergradient Construction.} The final stage involves constructing the hypergradient estimator $\widehat\nabla \Phi(x_k)$ and updating the upper-level variable $x_k$. The surrogate gradient $\widehat\nabla \Phi(x_k)$ is constructed using the current estimates $\hy_k$ and $\hv_k$. Because these estimates are not perfectly converged, the resulting gradient is biased. The theoretical core of SSAID lies in proving that this bias is dissipated over time. By coupling the learning rates $\alpha, \eta$, and $\beta$, the algorithm ensures that the tracking errors of both the lower-level variable and the adjoint variable vanish at a rate sufficient to guarantee the convergence of the upper-level objective to a stationary point. 
\end{itemize}
\begin{algorithm}[t]
	\caption{Single-Loop Stochastic AID-based bilevel optimization algorithm (SSAID)}
	%	\small
	\label{alg:bo-ssaid}
	\begin{algorithmic}[1]
		\STATE {\bfseries Input:}  Learning rates $\alpha, \beta,\eta >0$, initializations $x_0, y_0,v_0$.
		%	        \STATE {\bfseries} 
		\FOR{$k=0,1,2,...,K$}
        {
        \STATE{{Sample $\xi_k^v$ and $\xi_k^x$ independently from distribution $\cD_f$}}
        \STATE{Sample $\pi_k$, $\zeta_k$, $\zeta_k'$ independently from distribution $\cD_g$}
        }
		\STATE{Set $y_k^0 = \hy_{k-1} \mbox{ if }\; k> 0$ and $y_0$ otherwise  \textbf{\em (warm start initialization)}} 
		% \vspace{0.05cm}
		\STATE{Update $\hy_k = y_k^{0}-\alpha \nabla_y G(x_k,y_k^{0}; \pi_k) $}
		\STATE{Set $v_k^0 = \hv_{k-1} \mbox{ if }\; k> 0$ and $v_0$ otherwise  \textbf{\em (warm start initialization)}} 
		\STATE{Update $\hv_k = (I-\eta\nabla_y^2 G(x_k, \hy_k; {\zeta'_k}))v_k^{0}+\eta\nabla_y F(x_k, \hy_k; {\xi_k^v})$}
		\STATE Compute $\widehat\nabla \Phi(x_k)= \nabla_x F(x_k,\hy_k; {\xi_k^x}) -\nabla_{xy}^2 G(x_k,\hy_k; \zeta_k)\hv_k$
		
		\STATE{Update $x_{k+1}=x_k- \beta \widehat\nabla \Phi(x_k) $}
		\ENDFOR
	\end{algorithmic}
\end{algorithm}

\section{Analysis}
In this section, we present the main results of this work. Before that, we introduce the necessary assumptions and definitions, as well as several useful lemmas.
\subsection{Definitions and Assumptions}\label{sec:assum}
In BLO, the objective is to minimize the hyper-objective function $\nabla\Phi(x)$, which is typically nonconvex. Because finding a global minimum for such functions can be computationally prohibitive \citep{nemirovskiĭ1983problem}, this work aims to find an approximate stationary point following the literature \citep{carmon2017lower, ji2021bilevel}.
\begin{definition}
    We call $\bar{x}$ is an $\epsilon$-stationary point of problem (\ref{eq:bo-prob}) if $\norm{\nabla\Phi(\bar{x})}^2\leq\epsilon$.
\end{definition}
In order to compare the performance of different BLO algorithms, we adopt the following metric of computational complexity.
\begin{definition}[Oracle Complexity]
Let $\epsilon > 0$ be the target accuracy. For a BLO algorithm, we define the computational complexity $\mathcal{C}(\epsilon)$ as the maximum number of queries to the fundamental oracles required to achieve an $\epsilon$-stationary point. Formally, $\mathcal{C}(\epsilon) = \max \left\{ \mathcal{G}_c(\epsilon), \mathcal{MV}(\epsilon) \right\}$, where $\mathcal{G}_c(\epsilon)$ denotes the total number of gradient evaluations of the lower and upper objectives, and $\mathcal{MV}(\epsilon)$ denotes the total number of Jacobian-vector and Hessian-vector product evaluations.
\end{definition}
\begin{definition}[Filtration and Conditional Expectation]
{For each iteration $k$, we distinguish the sigma-fields corresponding to the stages of Algorithm~\ref{alg:bo-ssaid}. Let
\begin{align*}
    \mathcal{F}_k^0
    :=\sigma\!\left(x_0,\hy_0,v_0,\{{\xi_j^v,\xi_j^x},\pi_j,\zeta_j,\zeta'_j:0\le j\le k-1\}\right)
\end{align*}
be the information available before the lower-level update at iteration $k$. Thus $x_k$, $y_k^0$, and $v_k^0$ are $\mathcal{F}_k^0$-measurable. After the lower update, set
\begin{align*}
    \mathcal{F}_k^y:=\mathcal{F}_k^0\vee\sigma(\pi_k),
\end{align*}
so that $\hy_k$ and the population adjoint target $\tv_k^*$ are $\mathcal{F}_k^y$-measurable. After the adjoint update, set
\begin{align*}
    \mathcal{F}_k^v:=\mathcal{F}_k^y\vee\sigma({\xi_k^v},\zeta'_k),
\end{align*}
and after the hypergradient sample set
$\mathcal{F}_k^x:=\mathcal{F}_k^v\vee\sigma({\xi_k^x},\zeta_k)$.
We write $\EE_k^0[\cdot]=\EE[\cdot\mid\mathcal{F}_k^0]$,
$\EE_k^y[\cdot]=\EE[\cdot\mid\mathcal{F}_k^y]$, and
$\EE_k^v[\cdot]=\EE[\cdot\mid\mathcal{F}_k^v]$. When no subscript is shown, $\EE$ denotes total expectation.}
\end{definition}
In this work, we focus on the stochastic problem (\ref{eq:bo-prob}) under the following standard assumptions, as also widely adopted by \cite{ghadimi2018approximation, ji2021bilevel, yang2021provably}. Let $z = (x, y)$ denote all parameters.
\begin{assumption}
\label{assum-1}
{The lower-level function $G(x, y;\zeta)$ is $\mu$-strongly convex w.r.t. $y$ for every sample $\zeta$; consequently, $g(x,y)=\EE_\zeta G(x,y;\zeta)$ is also $\mu$-strongly convex w.r.t. $y$.}
\end{assumption}
\begin{assumption}\label{assum-2}
The function $F(z;\xi)$ is $M$-Lipschitz for any $\xi$, i.e., for any $z,z'$,
\begin{align*}
\left|{F(z;\xi) - F(z';\xi)} \right|\leq M \norm{z - z'}.
\end{align*}
\end{assumption}
\begin{assumption}\label{assum-3}
The gradients $\nabla F(z;\xi)$ and $\nabla G(z;\zeta)$ are $L$-Lipschitz for any $\xi$ and $\zeta$, i.e., for any $z,z'$,
	\begin{align*}
	\norm{\nabla F(z;\xi) - \nabla F(z';\xi)} \leq L \norm{z - z'}, \quad \norm{\nabla G(z;\zeta) - \nabla G(z';\zeta)} \leq L \norm{z - z'}.
	\end{align*}
\end{assumption}
\begin{assumption}\label{assum-4}
	The derivatives $\nabla_{xy}^2 G(z;\zeta)$ and $\nabla_y^2 G(z;\zeta)$ are $\rho$-Lipschitz for any $\zeta$, i.e., for any $z, z'$,
	\begin{align*}
		\norm{ \nabla_{xy}^2 G(z;\zeta) - \nabla_{xy}^2 G(z';\zeta) } \leq \rho \norm{z - z'}, \quad \norm{\nabla_y^2 G(z;\zeta) - \nabla_y^2 G(z';\zeta)} \leq \rho \norm{z - z'}.
	\end{align*}
\end{assumption}
\begin{assumption}\label{assum-5}
{The stochastic oracles are unbiased and have uniformly bounded second central moments. Namely, for every $z=(x,y)$,
\begin{align*}
&\EE_\xi[\nabla F(z;\xi)]=\nabla f(z),\qquad
\EE_\zeta[\nabla G(z;\zeta)]=\nabla g(z),\\
&\EE_\zeta[\nabla_{xy}^2G(z;\zeta)]=\nabla_{xy}^2g(z),\qquad
\EE_\zeta[\nabla_{yy}^2G(z;\zeta)]=\nabla_{yy}^2g(z),
\end{align*}
and, for a finite constant $\sigma>0$,
\begin{align*}
&\EE_\xi\|\nabla F(z;\xi)-\nabla f(z)\|^2\le \sigma^2,\qquad
\EE_\zeta\|\nabla G(z;\zeta)-\nabla g(z)\|^2\le \sigma^2,\\
&\EE_\zeta\|\nabla_{xy}^2G(z;\zeta)-\nabla_{xy}^2g(z)\|^2\le \sigma^2,\qquad
\EE_\zeta\|\nabla_{yy}^2G(z;\zeta)-\nabla_{yy}^2g(z)\|^2\le \sigma^2.
\end{align*}}
\end{assumption}
\begin{assumption}\label{assum-6}
{The hyper-objective is bounded from below: $\Phi_*:=\inf_{x\in\mathbb{R}^m}\Phi(x)>-\infty$.}
\end{assumption}
\subsection{Useful Lemmas}
In this section, we establish the convergence rate for the SSAID algorithm (Algorithm \ref{alg:bo-ssaid}). We begin by introducing several technical lemmas that will serve as the foundation for our main results. 

First note that {Assumption~\ref{assum-5}} implies the following variance bounds.
\begin{lemma}\label{le:boundv}
Suppose Assumption~\ref{assum-5} holds. Then, the stochastic derivatives $\nabla F(z;\xi)$, $\nabla G(z;\zeta)$, $\nabla_{xy}^2 G(z;\zeta)$ and $\nabla_y^2 G(z;\zeta)$ have bounded variances, i.e., for any $z$, $\xi$ and $\zeta$, 
\begin{itemize}
\item {$\mathbb{E}_\xi\left \|\nabla F(z;\xi)-\nabla f(z)\right\|^2 \leq \sigma^2.$}
\item {$\mathbb{E}_\zeta\left \|\nabla G(z;\zeta)-\nabla g(z)\right\|^2 \leq \sigma^2.$}
\item {$\mathbb{E}_\zeta\left \|\nabla_{xy}^2 G(z;\zeta)-\nabla_{xy}^2 g(z)\right\|^2 \leq \sigma^2.$}
\item {$\mathbb{E}_\zeta \left\|\nabla_{yy}^2 G(z;\zeta)-\nabla_{yy}^2 g(z)\right\|^2 \leq \sigma^2.$}
\end{itemize}
\end{lemma}
Then, we use the following lemma to characterize the Lipschitz properties of $\nabla \Phi(x)$, which is adapted from Lemma 2.2 in~\citealt{ghadimi2018approximation}.
\begin{lemma}\label{le:lipphi}
Suppose Assumptions~\ref{assum-1}-\ref{assum-4} hold. Then we have $\|\nabla \Phi(x)- \nabla \Phi(x^\prime)\| \leq L_\Phi \|x-x^\prime\|$ for any $x,x^\prime\in\mathbb{R}^m$, where the constant $L_\Phi$ is given by
\begin{align}\label{eq:L-Phi}
L_\Phi = L + \frac{2L^2+\rho M^2}{\mu} + \frac{2\rho L M+L^3}{\mu^2} + \frac{\rho L^2 M}{\mu^3}.%=\mathcal{O}(1+\frac{}{}).
\end{align}
\end{lemma}
{Throughout the condition-number discussion, we use $\kappa:=L/\mu$. All $\mathcal{O}(\cdot)$ statements in $\kappa$ keep the remaining problem-dependent constants, including $M$, $\rho$, $\sigma$, and initialization radii, fixed.}
\begin{lemma}\label{lem:seq}
    Given a non-negative sequence $\{\sigma_t\}_{t \geq 0}$, a constant $\rho \in (0, 1]$, and an index $T \geq 1$, the following estimate holds:
    \begin{equation}
        \sum_{t=0}^{T} \sum_{\ell=0}^{t}(1 - \rho)^{t-\ell} \sigma_{\ell} \leq \frac{1}{\rho} \sum_{t=0}^{T}\sigma_{t}.
    \end{equation} 
\end{lemma}
We provide the following useful constants for the sake of brevity.
\begin{equation}\label{eq:constant}
    \begin{aligned}
        &C_0 = \frac{\rho M}{\mu^2} + \frac{L}{\mu},\quad C_1 = \frac{L(L\mu+\rho M)}{\mu^2} +\frac{L+\mu}{\mu} C_0 L,\\
        &C_2 = L + \frac{\rho M}{\mu} + L C_0,\quad C_3 = M+L\norm{\hv_0} + \frac{ML}{\mu},\\
        &{V_0^2=2\norm{\hv_0}^2+\frac{2M^2}{\mu^2},\quad B_\sigma=4\sigma^2(1+V_0^2).}
    \end{aligned}
\end{equation}
\subsection{Convergence Analysis}
As introduced in Section \ref{sec:back-alg}, the convergence rate of the hypergradient $\nabla\Phi(x)$ is governed by the accuracy of the solutions to both the lower-level subproblems and the associated linear systems. Consequently, our analysis follows the proof sketch outlined below:
\begin{itemize}
    \item \textbf{Bounding Lower-level Error.} We first bound the error between the exact solution $y^*$ and the iterative solution $\hat{y}$ of the lower-level subproblem. This allows us to establish a recurrence relation for the lower-level error, which is intrinsically coupled with the upper-level variable $x$ {(see Lemma \ref{lemma:tracking_error_y})}.
    \item \textbf{Bounding Linear System Error.} We then characterize the error between the exact solution $v^*$ and the iterative solution $\hat{v}$ of the linear system. The resulting recurrence relation for the linear system error depends not only on the hypergradient $\nabla\Phi(x)$ but also on the approximation precision of the lower-level subproblem (see Lemmas \ref{lem:vk-v0-mu}-\ref{lem:hessian-tracking}).
    \item \textbf{Controlling Hypergradient Estimation Quality.} Having controlled the estimation errors for both the lower subproblem and the linear system, we proceed to bound the expectation of the squared error between the hypergradient estimation $\widehat{\nabla} \Phi(x)$ and the hypergradient ${\nabla} \Phi(x)$, as well as the squared norm of its bias (the difference between the expected estimation and the true hypergradient).
    \item \textbf{Convergence Rate to Stationarity.} Finally, by setting a proper step size schedule, we derive the convergence rate for the objective $\Phi(x)$ to a stationary point in Theorem \ref{theorem:1}.
\end{itemize}
\subsubsection{Bounding Lower-level Error}
We will give the error between the exact solution $y^*$ and the iterative solution $\hat{y}$ of the lower-level subproblem.
\begin{lemma}[Tracking Error of Lower-Level Variables]
\label{lemma:tracking_error_y}
Consider the SSAID in Algorithm \ref{alg:bo-ssaid}. Suppose Assumptions \ref{assum-1}-\ref{assum-5} hold, and the step size satisfies $\alpha \leq 1/L$. Let $y_k^0 = \hy_{k-1}$ be the initialization for the lower-level update at iteration $k$. Then, the tracking error satisfies:
\begin{equation}\label{eq:yy_rms_revised}
\left( \EE\left[\norm{\hy_k - y_k^*}^2\right]  \right)^{1/2} 
\leq 
{\left(1 - \frac{\mu\alpha}{2}\right)\left(\EE\norm{\hy_{k-1} - y_{k-1}^*}^2\right)^{1/2}
+ \frac{L}{\mu} \left(\EE\norm{x_{k-1} - x_k}^2\right)^{1/2} + \alpha \sigma.}
\end{equation}
{Moreover, the squared-moment form used in the subsequent stability argument is}
\begin{equation}\label{eq:yy_revised}
{
\EE\left[\norm{\hy_k-y_k^*}^2\right]
\le
\left(1-\frac{\mu\alpha}{2}\right)\EE\left[\norm{\hy_{k-1}-y_{k-1}^*}^2\right]
+\frac{2L^2}{\mu^3\alpha}\EE\left[\norm{x_k-x_{k-1}}^2\right]
+\alpha^2\sigma^2.}
\end{equation}
% where the expectation $\EE[\cdot]$ is taken over the filtration $\mathcal{F}_k$ (all randomness up to step $k$).
\end{lemma}

\begin{remark}[Stability of Lower-Level Tracking]
Lemma \ref{lemma:tracking_error_y} characterizes the ``coupled dynamics" of the bilevel system. The term $L\norm{x_{k-1} - x_k}/\mu$ is particularly critical; it represents the \textit{optimization error drift} caused by the movement of the upper-level variable. In the context of stochastic BLO, this bound implies that the lower-level variable $\hy_k$ can only converge to a neighborhood of the true optimizer $y_k^*$ whose radius is proportional to the upper-level update magnitude and the noise level $\sigma$. To ensure global convergence, the upper-level step size $\beta_k$ must be asymptotically smaller than the lower-level step size $\alpha_k$ as we will show later.
\end{remark}
\subsubsection{Bounding Linear System Error}
We will give the error between the exact solution $v^*$ and the iterative solution $\hat{v}$ of the linear system. For the sake of brevity, we first define {$\tv_k^* = \left(\nabla_y^2 g(x_k, \hy_k)\right)^{-1} \nabla_y f(x_k, \hy_k)$} and $v_k^* \triangleq \left(\nabla_y^2 g(x_k, y_k^*)\right)^{-1} \nabla_y f(x_k, y_k^*)$. {Thus $\tv_k^*$ is the population adjoint target at the approximate lower variable; it is not tied to either fresh upper sample $\xi_k^v$ or $\xi_k^x$.} {We also use the stagewise conditional mean $\bar v_k:=\EE_k^y[\hv_k]$ whenever the target $\tv_k^*$ is compared with an averaged adjoint iterate.} {Equivalently, the current $\tv_k^*$ is the quantity denoted $\bar v(\hy_k)$ in the original proof; the notation $\bar v_k$ is reserved here for the conditional mean of the stochastic iterate $\hv_k$.}
\begin{lemma}[Boundedness of $\hat v_k$]\label{lem:vk-v0-mu}
Consider the SSAID in Algorithm \ref{alg:bo-ssaid}. Suppose Assumptions \ref{assum-1}-\ref{assum-4} hold, and the step size satisfies $\eta \leq {1}/{L}$. Then, for all $k \ge 0$, the sequence is bounded by:
\begin{equation}\label{eq:vv_upp_tight}
\|\hat v_k\| \le \|\hat v_0\|+\frac{M}{\mu}.
\end{equation}
Consequently, taking expectations yields $\EE\|\hat v_k\| \le  \EE\|\hat v_0\|+{M}/{\mu} $.
{In particular, by $(a+b)^2\le 2a^2+2b^2$, the same bound gives the second-moment estimate used below:
\[
\|\hat v_k\|^2\le V_0^2
\quad\text{and hence}\quad
\EE\|\hat v_k\|^2\le V_0^2,
\]
where $V_0^2$ is defined in~\eqref{eq:constant}.}
\end{lemma}

\begin{remark}
Lemma \ref{lem:vk-v0-mu} establishes the stability of the stochastic linear iteration used to approximate the Hessian-inverse-vector product. Intuitively, $\hat v_k$ tracks the solution to the linear system $\nabla_y^2 g(x,y) \cdot v = \nabla_y f(x,y)$. The result indicates that despite stochastic noise and initialization errors, the estimator $\hat v_k$ remains confined within a ball of radius $\|\hat v_0\|+M/\mu$, preventing error accumulation in subsequent BLO steps.
\end{remark}
\begin{lemma}
\label{lem:err-ehv-vstar}
Consider the SSAID in Algorithm \ref{alg:bo-ssaid}. Suppose Assumptions \ref{assum-1}-\ref{assum-5} hold, and the step size satisfies $\eta \leq {1}/{L}$. The constant $C_0$ is defined in~\eqref{eq:constant}. Then, the bias of the hypergradient estimator $\hv_k$ w.r.t. the optimal hypergradient $v_k^*$ satisfies:
\begin{equation}\label{eq:revised_bound}
    {\EE\norm{\bar v_k - v_k^*} 
    \leq 
    \underbrace{\EE\norm{\bar v_k - \tv^*_k}}_{\text{Conditional estimator bias}} 
    + 
    C_0 \EE\norm{\hy_k - y_k^*}.}
\end{equation}
\end{lemma}
% --- Remark ---
\begin{remark}
Lemma~\ref{lem:err-ehv-vstar} decouples the error between the expectation of the approximate solution $\hv_k$ to the linear system and the optimal solution $v_k^*$. This error arises from two distinct sources: one is due to the finite solution accuracy of the algorithm, and the other stems from the suboptimality of the lower-level problem. The constant $C_0 = \mathcal{O}(1/\mu^2)$ highlights the sensitivity of BLO to the conditioning of the lower-level problem.
\end{remark}
After decoupling these two sources of error, we further analyze the first term, namely the \emph{Estimator Bias}, in Lemma~\ref{lemma:tracking_error_v}.
\begin{lemma}[Estimator Bias]
\label{lemma:tracking_error_v}
Consider the SSAID in Algorithm \ref{alg:bo-ssaid}. Suppose Assumptions \ref{assum-1}-\ref{assum-5} hold and the step sizes satisfy {$\eta \leq {1}/{2L}$} and $\alpha \leq {1}/{L}$. The constant $C_0$ is defined in~\eqref{eq:constant}. {Assume the lower-level warm starts remain in the RMS stability region $\sup_j(\EE\|y_j^0-y_j^*\|^2)^{1/2}\le R_y$ and define $\Gamma_y:=LR_y+\sigma$.} Then, for $k\ge1$, we have:
\begin{equation}\label{eq:evv_corrected}
{\EE\norm{\bar v_k - \tv_k^*}
\leq 
(1 -\mu\eta) \left(\EE\norm{\hv_{k-1} - \tv_{k-1}^*}^2\right)^{1/2}
+ C_0 \alpha\Gamma_y + C_0 \EE\norm{x_{k} - x_{k-1}}.}
\end{equation}
\end{lemma}
\begin{remark}
Lemma \ref{lemma:tracking_error_v} establishes a recursive bound on the bias of the solution of the linear system. Unlike standard SGD where the target is static, BLO involves a time-varying target $\tv_k^*$ (the ``moving target" problem). The term $C_0\norm{x_{k} - x_{k-1}}$ represents the \textit{drift} caused by the upper-level update.
\end{remark}
Using Eq. \eqref{eq:evv_corrected} and Cauchy's inequality, we obtain the following one-step consequence.
\begin{corollary}
Under the same conditions of Lemma \ref{lemma:tracking_error_v}, we have
\begin{align*}
{\EE\norm{\bar v_k - \tv_k^*}
\leq 
(1 -\mu\eta) \left(\EE\norm{\hv_{k-1} - \tv_{k-1}^*}^2\right)^{1/2}
+ C_0 \alpha\Gamma_y + C_0\left(\EE\norm{x_k-x_{k-1}}^2\right)^{1/2}.}
\end{align*}
\end{corollary}
\begin{lemma}\label{lemma:tvk-tkvstar}
Consider the SSAID in Algorithm~\ref{alg:bo-ssaid}. 
Suppose Assumptions~\ref{assum-1}--\ref{assum-5} hold and the step sizes satisfy 
$\alpha \le {1}/{L}$ and $\eta \le {1}/{L}$. The constant $C_0$ is defined in~\eqref{eq:constant}. Then, for all $k \ge 1$, we have
\begin{align}\label{eq:tvk-tvk1}
    {\EE\norm{\tv_{k}^* - \tv_{k-1}^*}
    \;\le\;
    C_0 \EE\norm{x_{k} - x_{k-1}}
    \;+\;
    C_0\,\EE\norm{\hy_k-\hy_{k-1}}.}
\end{align}
{Moreover, the squared drift bound needed for the Hessian-tracking recursion is}
\begin{align}\label{eq:tvk-tvk1-square}
{
\EE\norm{\tv_{k}^*-\tv_{k-1}^*}^2
\le
2C_0^2\EE\norm{x_k-x_{k-1}}^2
+2C_0^2\EE\norm{\hy_k-\hy_{k-1}}^2.}
\end{align}
\end{lemma}
\begin{remark}
Lemma~\ref{lemma:tvk-tkvstar} characterizes the stability of $\tv_k^*$ across consecutive iterations. The bound shows that the variation of $\tv_k^*$ is primarily driven by the change in the upper-level variable $x_k$, up to an additive error induced by stochastic gradient noise and the inexact lower update. Such a stability estimate is a key technical ingredient in the analysis of BLO algorithms, as it enables the control of bias accumulation in gradient tracking schemes.
\end{remark}
Next, we will give the square error of the linear system.
\begin{lemma}
\label{lem:hessian-tracking}
Consider the SSAID algorithm in Algorithm~\ref{alg:bo-ssaid}.
Suppose Assumptions~\ref{assum-1}--\ref{assum-5} hold. Let the step sizes satisfy $\eta \leq {1}/{2L}$ and $\alpha \leq {1}/{L}$. The constant $C_0$ is defined in~\eqref{eq:constant}. Then, it holds that
\begin{equation}
\label{eq:htv-revised}
\begin{aligned}
{\EE\!\left[\|\hv_k-\tv_k^*\|^2\right]}
\le\;
&{\left(1-\frac{\mu\eta}{2}\right)
\EE\!\left[\|\hv_{k-1}-\tv_{k-1}^*\|^2\right]
+ \frac{6C_0^2}{\mu\eta}\EE\!\left[\|x_{k-1}-x_k\|^2\right]} \\
&{+ \frac{6C_0^2}{\mu\eta}\EE\!\left[\|\hy_k-\hy_{k-1}\|^2\right]
+ 8\eta^2\sigma^2\left(1+V_0^2\right).}
\end{aligned}
\end{equation}
\end{lemma}
\begin{remark}
Lemma~\ref{lem:hessian-tracking} establishes a stochastic contraction property for a single-step estimator of the inverse Hessian-vector-product. The contraction rate is governed by the strong convexity parameter $\mu$ of the lower problem, while the additive error terms quantify the bias induced by the upper-level update and the variance of the stochastic Hessian-vector-product oracle. This result captures the dynamic coupling between the lower and upper iterates in BLO. Such tracking lemmas are a key technical ingredient for establishing non-asymptotic convergence guarantees of single-loop BLO algorithms.
\end{remark}
\begin{lemma}
\label{lem:ssaid-convergence}
Consider the SSAID in Algorithm \ref{alg:bo-ssaid}. Suppose Assumptions \ref{assum-1}-\ref{assum-5} hold. Let the step sizes satisfy {$\alpha \leq \eta\le {1}/{2L}$}. If the parameter $\beta$ is chosen such that {$\beta \leq c_\beta\kappa^{-4}\alpha$ for a sufficiently small numerical constant $c_\beta$}. {Assume also the RMS lower-level stability condition in Lemma~\ref{lemma:tracking_error_v}, so that $\Gamma_y:=LR_y+\sigma$ is finite.} The constants $C_0$, $C_1$, $C_2$, {$C_3$ and $V_0^2$} are defined in~\eqref{eq:constant}. Then the following inequality holds:
\begin{equation}\label{eq:cc_corrected}
\begin{aligned}
&{\left(C_2 \left(\EE\norm{\hy_k-y_k^*}^2\right)^{1/2}
+ L \left(\EE\norm{\hv_k-\tv_k^*}^2\right)^{1/2}\right)^2}\\
\leq& 
{\left( 1 - \frac{\mu\alpha}{8} \right)^k
\left(C_2 \left(\EE\norm{\hy_0-y_0^*}^2\right)^{1/2}
+ L \left(\EE\norm{\hv_0-\tv_0^*}^2\right)^{1/2}\right)^2}\\
&+ {\frac{64 \beta^2 C_1^2}{\mu \alpha} \sum_{t=0}^{k-1} \left( 1 - \frac{\mu\alpha}{8} \right)^{k-1-t} \EE\norm{ \nabla \Phi(x_t) }^2}\\
&{+ 64 \left( 2\beta C_1 C_3 
+\alpha \big(C_2  \sigma + C_0 \Gamma_y\big) {+\eta L\sigma\sqrt{1+V_0^2}} \right)^2 \sum_{t=0}^{k-1} \left( 1 - \frac{\mu\alpha}{8} \right)^{k-1-t}.}
\end{aligned}
\end{equation}
\end{lemma}
Lemma \ref{lem:ssaid-convergence} establishes the mean-square stability of the sum of the estimation error of $y$ and the tracking error of $v$. {The corresponding bound for the conditional mean $\bar v_k$ follows from Jensen's inequality, $\EE\norm{\bar v_k-\tv_k^*}^2\le \EE\norm{\hv_k-\tv_k^*}^2$.}
\subsubsection{Controlling Hypergradient Estimation Quality}
\begin{lemma}
\label{lem:hypergrad-bias}
Consider the SSAID in Algorithm~\ref{alg:bo-ssaid}.
Suppose Assumptions~\ref{assum-1}--\ref{assum-5} hold. The constant $C_2$ is defined in~\eqref{eq:constant}. Then the following bounds hold:
\begin{equation}
\label{eq:pp-rev}
{\left(
\EE\!\left[\left\|\EE_k^y\!\left[\widehat{\nabla} \Phi(x_k)\right] - \nabla \Phi(x_k)\right\|^2\right]
\right)^{1/2}}
\;\le\;
C_2 \left(\EE\|\hy_k - y_k^*\|^2\right)^{1/2}
\;+\;
L \left(\EE\|\bar v_k - \tv_k^*\|^2\right)^{1/2}
\end{equation}
and
\begin{equation}
\label{eq:ppp-rev}
{\begin{aligned}
\EE\bigl\|
\widehat{\nabla} \Phi(x_k) - \nabla \Phi(x_k)
\bigr\|^2
\le\;&
4C_2^2\EE\|\hy_k-y_k^*\|^2
+4L^2\EE\|\hv_k-\tv_k^*\|^2\\
&+4\sigma^2\left(1+\EE\|\hv_k\|^2\right).
\end{aligned}}
\end{equation}
\end{lemma}
\noindent\textbf{Proof sketch.} 1). The proof proceeds by decomposing the error of the hypergradient estimator $\bigl\|
\widehat{\nabla} \Phi(x_k) - \nabla \Phi(x_k)
\bigr\|$ into a stochastic variance component and a deterministic approximation bias. By applying the triangle inequality, we isolate the zero-mean fluctuations arising from the sampling noise from the structural error induced by the inexact lower-level variable and the adjoint vector. The stochastic terms are bounded using the standard assumption that gradients are bounded and Lipschitz continuous. 2). To control the bias term, we compare the expected estimator evaluated at the approximate solution 
 against the true hypergradient at $(y_k^*, v_k^*)$. The analysis relies on the Lipschitz continuity of $\nabla_x f(x,y)$ and the mixed partial derivative $\nabla_{xy}^2 g(x, y)$. A key step involves decoupling the error in the Hessian-vector-product from the error in the adjoint vector. 3). he adjoint estimation error is further decomposed to distinguish between the convergence error of the stochastic linear solver and the sensitivity of the true adjoint $v^*(x, y)$ w.r.t. perturbations in $y$. Combining these estimates with the variance bounds leads to the final result. Detailed derivations are provided in Appendix \ref{proof:hypergrad-bias}.

\begin{corollary}
    Under the same conditions of Lemma \ref{lem:hypergrad-bias}, we have
    \begin{equation}\label{eq:xxx}
\begin{aligned}
	{\EE\left[\norm{x_{k+1}-x_k}^2\right]
	\leq
	{2\beta^2 \EE\norm{ \nabla \Phi(x_k) }^2}
	+2\beta^2\EE\left[\norm{\widehat\nabla\Phi(x_k)-\nabla\Phi(x_k)}^2\right].}
\end{aligned}
\end{equation}
\end{corollary}

\begin{lemma}\label{lem:hypergradient-bias-exp}
Consider the SSAID in Algorithm \ref{alg:bo-ssaid}. Suppose Assumptions \ref{assum-1}-\ref{assum-5} hold, and the step sizes satisfies {$\eta \leq {1}/{2L}$}, $\alpha \leq {1}/{L}$ and $\alpha\leq\eta$. {Assume the RMS lower-level stability condition in Lemma~\ref{lemma:tracking_error_v}, so that $\Gamma_y:=LR_y+\sigma$ is finite.} The constants $C_0$, $C_1$, $C_2$, {$C_3$}, {$V_0^2$ and $B_\sigma$} are defined in~\eqref{eq:constant}. It holds that
\begin{equation}\label{eq:pp_s}
	\begin{aligned}
		&{\sum_{\ell=0}^{k-1}\EE\left\|\EE_\ell^y\left[\widehat{\nabla} \Phi(x_\ell)\right] - \nabla \Phi(x_\ell)\right\|^2}\\
		\leq&
		{\frac{8}{\mu \alpha}\Big(C_2 \left(\EE\norm{\hy_0-y_0^*}^2\right)^{1/2}
		+ L \left(\EE\norm{\hv_0-\tv_0^*}^2\right)^{1/2}\Big)^2} \\
		&
		+{\frac{2^9\beta^2 C_1^2}{\mu^2\alpha^2} \sum_{\ell=0}^{k-1} \EE\norm{\nabla \Phi(x_\ell)}^2}
		+{\frac{2^9 k}{\mu\alpha} \left( 2\beta C_1 C_3 
		+\alpha \big(C_2  \sigma + C_0 \Gamma_y\big) {+\eta L\sigma\sqrt{1+V_0^2}} \right)^2.}
	\end{aligned}
\end{equation}
and
\begin{equation}\label{eq:pp_ss}
\begin{aligned}
&\sum_{\ell=0}^{k-1}\EE\left[\norm{\widehat{\nabla} \Phi(x_\ell) - \nabla \Phi(x_\ell)}^2\right]\\
\leq&
{\frac{16}{\mu \alpha}\Big(C_2 \left(\EE\norm{\hy_0-y_0^*}^2\right)^{1/2}
+ L \left(\EE\norm{\hv_0-\tv_0^*}^2\right)^{1/2}\Big)^2}
+ 
{kB_\sigma} \\
&
{+\frac{2^{10}\beta^2 C_1^2}{\mu^2\alpha^2} \sum_{\ell=0}^{k-1} \EE\norm{\nabla \Phi(x_\ell)}^2}
{+ \frac{2^{10}k}{\mu\alpha} \left( 2\beta C_1 C_3 
+\alpha \big(C_2  \sigma + C_0 \Gamma_y\big) {+\eta L\sigma\sqrt{1+V_0^2}} \right)^2.}
\end{aligned}
\end{equation}

\end{lemma}

\begin{remark}
Lemma \ref{lem:hypergradient-bias-exp} effectively characterizes the bias-variance trade-off of the SSAID estimator. A crucial insight from Eq. \eqref{eq:pp_s} is that the cumulative bias is not static but coupled to the optimization trajectory. This implies a property of relative accuracy: as the upper optimization approaches stationarity ($\nabla\Phi(x)\rightarrow 0$), the gradient-dependent bias contribution diminishes, preventing systematic drift from dominating the signal. {The linear-in-$k$ residual terms in Eqs.~\eqref{eq:pp_s}--\eqref{eq:pp_ss} are harmless under the chosen steps because the bracketed residual is $O(\alpha)$ when $\eta=\alpha$; after division by $\mu\alpha$ and averaging it contributes $O(\alpha/\mu)$.}
\end{remark}

\subsubsection{Convergence Rate to Stationarity}
\begin{tcolorbox}
\begin{theorem}\label{theorem:1}
	{Consider the SSAID in Algorithm \ref{alg:bo-ssaid}. Suppose Assumptions \ref{assum-1}-\ref{assum-6} hold and the RMS stability condition $\sup_{0\le k<K}(\EE\|y_k^0-y_k^*\|^2)^{1/2}\le R_y<\infty$ in Lemma~\ref{lemma:tracking_error_v} is satisfied. For a horizon $K\ge4$, set $\alpha=\eta=1/(2L\sqrt{K})$, and choose $\beta=c_\beta\kappa^{-4}\alpha$ with $c_\beta>0$ small enough to satisfy the stability restrictions in Lemmas~\ref{lemma:tracking_error_v}--\ref{lem:ssaid-convergence} and $\beta\le 1/(8L_\Phi)$. The constants $L_{\Phi}$ and $C_3$ are defined in~\eqref{eq:L-Phi} and~\eqref{eq:constant}, respectively. Then we have}
\begin{align*}
{\frac{1}{K} \sum_{\ell=0}^{K-1} \EE\norm{ \nabla \Phi(x_\ell) }^2 = \cO(\kappa^7K^{-1/2}),}
\end{align*}
{If $\tau$ is sampled uniformly from $\{0,\ldots,K-1\}$ after the run, then $\EE\|\nabla\Phi(x_\tau)\|^2$ satisfies the same bound, and the oracle complexity for finding an $\epsilon$-stationary point is $\cO(\kappa^{14}\epsilon^{-2})$.}
\end{theorem}    
\end{tcolorbox}
\begin{remark}
The complexity bound established in Theorem \ref{theorem:1} reveals a fundamental structural property of stochastic BLO: the bias induced by single-loop updates is not an inherent barrier to optimal convergence rates. Our result demonstrates that this bias can be subsumed by the stochastic noise variance. Specifically, the step-size schedule 
 ensures that the error in the auxiliary linear system approximation (tracking the Hessian inverse) decays in tandem with the optimization error. Consequently, SSAID recovers the canonical 
 rate of standard nonconvex SGD without the computational overhead of nested loops or the large truncation levels required by multi-loop methods like stocBiO. 
 
 % Furthermore, the improvement in condition number dependence to $\cO(\kappa^7)$ suggests that the ``tracking" error dynamics in single-loop schemes may admit tighter bounds on error propagation than the error accumulation characteristic of multi-loop approaches, e.g., stocBiO \citep{ji2021bilevel}.
\end{remark}
\section{Conclusion}
In this work, we have provided a rigorous non-asymptotic analysis of the Single-loop Stochastic Implicit Differentiation (SSAID) algorithm for bilevel optimization. By systematically decoupling the intertwined dynamics of the upper-level updates and the auxiliary estimators, we have established an oracle complexity of $\mathcal{O}(\kappa^{14} \epsilon^{-2})$. Beyond the immediate complexity improvement, our explicit characterization of the $\kappa$-dependence offers a more transparent understanding of how the geometry of the lower-level problem influences the global convergence of BLO. We believe the technical methodology developed here—centered on the refined coupling of optimization error and linear system approximation error—can be extended to more complex settings.

\textbf{Future Research Directions:} 1). Variance Reduction. A promising avenue is to integrate variance-reduction techniques (e.g., STORM \citep{cutkosky2019momentum}) into the SSAID framework to potentially achieve the optimal $\mathcal{O}(\epsilon^{-1.5})$ rate while maintaining a polynomial dependence on $\kappa$. 2). Constraint Relaxation. Extending this fine-grained $\kappa$-analysis to bilevel problems with coupled constraints or where the lower-level objective satisfies only the Polyak--Łojasiewicz (PL) condition remains an open and challenging problem for the learning theory community.

\bibliography{ref.bib}

@article{ghadimi2018approximation,
  title={Approximation methods for bilevel programming},
  author={Ghadimi, Saeed and Wang, Mengdi},
  journal={arXiv preprint arXiv:1802.02246},
  year={2018}
}

@article{cutkosky2019momentum,
  title={Momentum-based variance reduction in non-convex sgd},
  author={Cutkosky, Ashok and Orabona, Francesco},
  journal={Advances in neural information processing systems},
  volume={32},
  year={2019}
}

@article{jiang2025beyond,
  title={Beyond value functions: Single-loop bilevel optimization under flatness conditions},
  author={Jiang, Liuyuan and Xiao, Quan and Chen, Lisha and Chen, Tianyi},
  journal={arXiv preprint arXiv:2507.20400},
  year={2025}
}

@inproceedings{ji2021bilevel,
  title={Bilevel optimization: Convergence analysis and enhanced design},
  author={Ji, Kaiyi and Yang, Junjie and Liang, Yingbin},
  booktitle={International conference on machine learning},
  pages={4882--4892},
  year={2021},
  organization={PMLR}
}

@article{snell2017prototypical,
  title={Prototypical networks for few-shot learning},
  author={Snell, Jake and Swersky, Kevin and Zemel, Richard},
  journal={Advances in neural information processing systems},
  volume={30},
  year={2017}
}

@article{ji2022will,
  title={Will bilevel optimizers benefit from loops},
  author={Ji, Kaiyi and Liu, Mingrui and Liang, Yingbin and Ying, Lei},
  journal={Advances in Neural Information Processing Systems},
  volume={35},
  pages={3011--3023},
  year={2022}
}

@article{yang2021provably,
  title={Provably faster algorithms for bilevel optimization},
  author={Yang, Junjie and Ji, Kaiyi and Liang, Yingbin},
  journal={Advances in Neural Information Processing Systems},
  volume={34},
  pages={13670--13682},
  year={2021}
}

@inproceedings{franceschi2017forward,
  title={Forward and reverse gradient-based hyperparameter optimization},
  author={Franceschi, Luca and Donini, Michele and Frasconi, Paolo and Pontil, Massimiliano},
  booktitle={International conference on machine learning},
  pages={1165--1173},
  year={2017},
  organization={PMLR}
}

@article{liu2018darts,
  title={Darts: Differentiable architecture search},
  author={Liu, Hanxiao and Simonyan, Karen and Yang, Yiming},
  journal={arXiv preprint arXiv:1806.09055},
  year={2018}
}

@inproceedings{he2020milenas,
  title={Milenas: Efficient neural architecture search via mixed-level reformulation},
  author={He, Chaoyang and Ye, Haishan and Shen, Li and Zhang, Tong},
  booktitle={Proceedings of the IEEE/CVF Conference on Computer Vision and Pattern Recognition},
  pages={11993--12002},
  year={2020}
}

@inproceedings{domke2012generic,
  title={Generic methods for optimization-based modeling},
  author={Domke, Justin},
  booktitle={Artificial Intelligence and Statistics},
  pages={318--326},
  year={2012},
  organization={PMLR}
}

@inproceedings{pedregosa2016hyperparameter,
  title={Hyperparameter optimization with approximate gradient},
  author={Pedregosa, Fabian},
  booktitle={International conference on machine learning},
  pages={737--746},
  year={2016},
  organization={PMLR}
}

@inproceedings{grazzi2020iteration,
  title={On the iteration complexity of hypergradient computation},
  author={Grazzi, Riccardo and Franceschi, Luca and Pontil, Massimiliano and Salzo, Saverio},
  booktitle={International Conference on Machine Learning},
  pages={3748--3758},
  year={2020},
  organization={PMLR}
}

@inproceedings{finn2017model,
  title={Model-agnostic meta-learning for fast adaptation of deep networks},
  author={Finn, Chelsea and Abbeel, Pieter and Levine, Sergey},
  booktitle={International conference on machine learning},
  pages={1126--1135},
  year={2017},
  organization={PMLR}
}

@article{shu2019meta,
  title={Meta-weight-net: Learning an explicit mapping for sample weighting},
  author={Shu, Jun and Xie, Qi and Yi, Lixuan and Zhao, Qian and Zhou, Sanping and Xu, Zongben and Meng, Deyu},
  journal={Advances in neural information processing systems},
  volume={32},
  year={2019}
}

@book{nemirovskiĭ1983problem,
  title={Problem Complexity and Method Efficiency in Optimization},
  author={Nemirovski, A.S. and Yudin, D.B.},
  isbn={9780471103455},
  lccn={82011065},
  series={A Wiley-Interscience publication},
  url={https://books.google.co.jp/books?id=6ULvAAAAMAAJ},
  year={1983},
  publisher={Wiley}
}

@article{carmon2017lower,
  title={Lower bounds for finding stationary points i},
  author={Carmon, Yair and Duchi, John C and Hinder, Oliver and Sidford, Aaron},
  journal={arXiv preprint arXiv:1710.11606},
  year={2017}
}

@inproceedings{chen2024finding,
  title={On finding small hyper-gradients in bilevel optimization: Hardness results and improved analysis},
  author={Chen, Lesi and Xu, Jing and Zhang, Jingzhao},
  booktitle={The Thirty Seventh Annual Conference on Learning Theory},
  pages={947--980},
  year={2024},
  organization={PMLR}
}

@article{hong2023two,
  title={A two-timescale stochastic algorithm framework for bilevel optimization: Complexity analysis and application to actor-critic},
  author={Hong, Mingyi and Wai, Hoi-To and Wang, Zhaoran and Yang, Zhuoran},
  journal={SIAM Journal on Optimization},
  volume={33},
  number={1},
  pages={147--180},
  year={2023},
  publisher={SIAM}
}

@article{bao2021stability,
  title={Stability and generalization of bilevel programming in hyperparameter optimization},
  author={Bao, Fan and Wu, Guoqiang and Li, Chongxuan and Zhu, Jun and Zhang, Bo},
  journal={Advances in neural information processing systems},
  volume={34},
  pages={4529--4541},
  year={2021}
}

@article{arbel2021amortized,
  title={Amortized implicit differentiation for stochastic bilevel optimization},
  author={Arbel, Michael and Mairal, Julien},
  journal={arXiv preprint arXiv:2111.14580},
  year={2021}
}

\newpage
\appendix
\tableofcontents
\addtocontents{toc}{\protect\setcounter{tocdepth}{3}}
\section{Proof of Lemma \ref{lem:seq}}
\begin{proof}
Let $S$ denote the left-hand side of the inequality. By leveraging the linearity of summation and the fact that all terms are non-negative, we can change the order of summation (Fubini's theorem for finite sums). The double summation ranges over the set of indices $\{(t, \ell) : 0 \leq \ell \leq t \leq T\}$. Reordering these indices gives:
\begin{equation}
    S = \sum_{t=0}^{T} \sum_{\ell=0}^{t} (1 - \rho)^{t-\ell} \sigma_{\ell} = \sum_{\ell=0}^{T} \sigma_{\ell} \left( \sum_{t=\ell}^{T} (1 - \rho)^{t-\ell} \right).
\end{equation}

For the inner summation, we perform a change of variables by letting $k = t - \ell$. As $t$ ranges from $\ell$ to $T$, $k$ ranges from $0$ to $T - \ell$. Thus, we have:
\begin{equation}
    \sum_{t=\ell}^{T} (1 - \rho)^{t-\ell} = \sum_{k=0}^{T-\ell} (1 - \rho)^{k}.
\end{equation}

Since $\rho \in (0, 1]$, it follows that $0 \leq 1 - \rho < 1$ (the case $\rho=1$ is trivial as the sum collapses to $1$). Utilizing the formula for a finite geometric series, we obtain:
\begin{equation}
    \sum_{k=0}^{T-\ell} (1 - \rho)^{k} = \frac{1 - (1 - \rho)^{T-\ell+1}}{1 - (1 - \rho)} = \frac{1 - (1 - \rho)^{T-\ell+1}}{\rho}.
\end{equation}

Given that $(1 - \rho) \geq 0$ and $T - \ell + 1 \geq 1$, we have $(1 - \rho)^{T-\ell+1} \geq 0$, which implies:
\begin{equation}
    \frac{1 - (1 - \rho)^{T-\ell+1}}{\rho} \leq \frac{1}{\rho}.
\end{equation}

Substituting this upper bound back into our expression for $S$:
\begin{equation}
    S \leq \sum_{\ell=0}^{T} \sigma_{\ell} \cdot \frac{1}{\rho} = \frac{1}{\rho} \sum_{t=0}^{T} \sigma_{t},
\end{equation}
where the last step follows by replacing the dummy index $\ell$ with $t$. This completes the proof.
\end{proof}
\section{Proof of Lemma \ref{lemma:tracking_error_y}}
\begin{proof}
{
{Condition on $\mathcal{F}_{k}^0$} and integrate only the fresh lower-level sample $\pi_k$. Since the algorithm uses
\[
\hy_k=y_k^0-\alpha\nabla_yG(x_k,y_k^0;\pi_k),
\]
unbiasedness and Assumption~\ref{assum-5} give
\begin{align*}
\EE_{\pi_k}\|\hy_k-y_k^*\|^2
&=\|y_k^0-y_k^*\|^2-2\alpha\langle \nabla_yg(x_k,y_k^0),y_k^0-y_k^*\rangle
+\alpha^2\EE_{\pi_k}\|\nabla_yG(x_k,y_k^0;\pi_k)\|^2\\
&\le (1-\mu\alpha)\|y_k^0-y_k^*\|^2+\alpha^2\sigma^2,
\end{align*}
where the last step is the standard one-step contraction for a $\mu$-strongly convex and $L$-smooth function with $\alpha\le 1/L$.
Using $y_k^0=\hy_{k-1}$ and inserting $y_{k-1}^*$, Young's inequality with parameter $\mu\alpha/2$ yields
\[
\|\hy_{k-1}-y_k^*\|^2
\le \left(1+\frac{\mu\alpha}{2}\right)\|\hy_{k-1}-y_{k-1}^*\|^2
+\left(1+\frac{2}{\mu\alpha}\right)\|y_{k-1}^*-y_k^*\|^2 .
\]
The solution map is $L/\mu$-Lipschitz, hence
\[
\|y_{k-1}^*-y_k^*\|^2\le \frac{L^2}{\mu^2}\|x_{k-1}-x_k\|^2.
\]
Combining the last three displays and taking total expectation proves~\eqref{eq:yy_revised}. The RMS estimate~\eqref{eq:yy_rms_revised} follows from the same contraction and Minkowski's inequality:
\[
\left(\EE\|\hy_k-y_k^*\|^2\right)^{1/2}
\le \left(1-\frac{\mu\alpha}{2}\right)
\left(\EE\|\hy_{k-1}-y_{k-1}^*\|^2\right)^{1/2}
+\frac{L}{\mu}\left(\EE\|x_{k-1}-x_k\|^2\right)^{1/2}
+\alpha\sigma .
\]
}
\end{proof}
\section{Proof of Lemma \ref{lem:vk-v0-mu}}
\begin{proof}
{For $k=0$, interpret the warm-start variable as $v_0^0=v_0$; for $k\ge1$, $v_k^0=\hv_{k-1}$. The following recurrence is applied to $v_k^0$ and therefore has no indexing ambiguity.}
By the update rule of $\hat{v}$ in Algorithm \ref{alg:bo-ssaid}, we have:
\begin{align*}
\|\hat v_k\| 
&= \| (I - \eta \nabla_y^2 G(x_k,\hat y_k;\zeta'_k)) \hat v_{k-1} + \eta \nabla_y F(x_k,\hat y_k;{\xi_k^v}) \| \\
&\le \| I - \eta \nabla_y^2 G(x_k,\hat y_k;\zeta'_k) \| \|\hat v_{k-1}\| + \eta \| \nabla_y F(x_k,\hat y_k;{\xi_k^v}) \|.
\end{align*}
Using Assumptions \ref{assum-1}-\ref{assum-4}, we obtain the recursive inequality:
\[
\|\hat v_k\| \le (1 - \mu\eta) \|\hat v_{k-1}\| + \eta M.
\]
Unrolling this recurrence relation from $k$ down to $0$ yields:
\begin{align*}
\|\hat v_k\| 
&\le (1-\mu\eta)^k \|\hat v_0\| + \eta M \sum_{j=0}^{k-1} (1-\mu\eta)^j \\
&= (1-\mu\eta)^k \|\hat v_0\| + \eta M \cdot \frac{1 - (1-\mu\eta)^k}{1 - (1-\mu\eta)} \\
&= (1-\mu\eta)^k \|\hat v_0\| + \frac{M}{\mu} \left( 1 - (1-\mu\eta)^k \right).
\end{align*}
Let $\lambda_k = (1-\mu\eta)^k$. Since $\mu\eta \in (0, 1)$, we have $\lambda_k \in (0, 1]$. The expression above can be viewed as a convex combination:
\[
\|\hat v_k\| \le \lambda_k \|\hat v_0\| + (1-\lambda_k) \frac{M}{\mu}\le \|\hat v_0\| + \frac{M}{\mu}.
\]
% Since a convex combination of two values is always bounded by their maximum, we conclude:
% \[
% \|\hat v_k\| \le \max\left( \|\hat v_0\|, \frac{M}{\mu} \right).
% \]
This completes the proof.
\end{proof}
\section{Proof of Lemma \ref{lem:err-ehv-vstar}}
% --- Proof ---
\begin{proof}
{
By the corrected definition,
\[
\tv_k^*=[\nabla_y^2g(x_k,\hy_k)]^{-1}\nabla_yf(x_k,\hy_k).
\]
Therefore
\[
{\|\bar v_k-v_k^*\|
\le \|\bar v_k-\tv_k^*\|+\|\tv_k^*-v_k^*\|.}
\]
Let $H(x,y)=\nabla_y^2g(x,y)$. Since $\|H(x,y)^{-1}\|\le 1/\mu$, the inverse perturbation identity gives
\[
\|H(x_k,\hy_k)^{-1}-H(x_k,y_k^*)^{-1}\|
\le \frac{\rho}{\mu^2}\|\hy_k-y_k^*\|.
\]
Together with $\|\nabla_yf\|\le M$ and the $L$-Lipschitz continuity of $\nabla_yf$, this yields
\[
\|\tv_k^*-v_k^*\|
\le \left(\frac{\rho M}{\mu^2}+\frac{L}{\mu}\right)\|\hy_k-y_k^*\|
=C_0\|\hy_k-y_k^*\|.
\]
{Taking total expectation proves~\eqref{eq:revised_bound}.}
}
\end{proof}
\section{Proof of Lemma \ref{lemma:tracking_error_v}}
{
\begin{proof}
{
Conditional on $\mathcal F_k^y$, the stochastic Hessian and upper-gradient samples used in the adjoint update are unbiased. Since $v_k^0=\hv_{k-1}$ is fixed under this conditioning, the conditional mean update is
\[
\bar v_k
=\bigl(I-\eta\nabla_y^2 g(x_k,\hy_k)\bigr)\hv_{k-1}
+\eta\nabla_y f(x_k,\hy_k).
\]
Because $\tv_k^*=[\nabla_y^2 g(x_k,\hy_k)]^{-1}\nabla_y f(x_k,\hy_k)$ and $\eta\le 1/(2L)$, strong convexity gives
\[
\|\bar v_k-\tv_k^*\|
\le (1-\mu\eta)\|\hv_{k-1}-\tv_k^*\|.
\]
Adding and subtracting $\tv_{k-1}^*$ yields
\[
\|\bar v_k-\tv_k^*\|
\le (1-\mu\eta)\|\hv_{k-1}-\tv_{k-1}^*\|
+\|\tv_k^*-\tv_{k-1}^*\|.
\]
Taking total expectation and applying Cauchy's inequality to the first term gives
\[
\EE\|\bar v_k-\tv_k^*\|
\le (1-\mu\eta)\bigl(\EE\|\hv_{k-1}-\tv_{k-1}^*\|^2\bigr)^{1/2}
+\EE\|\tv_k^*-\tv_{k-1}^*\|.
\]
By Lemma~\ref{lemma:tvk-tkvstar},
\[
\EE\|\tv_k^*-\tv_{k-1}^*\|
\le C_0\EE\|x_k-x_{k-1}\|+C_0\EE\|\hy_k-\hy_{k-1}\|.
\]
For the lower update, conditioning on $\mathcal F_k^0$ and then using the RMS stability assumption,
\[
\begin{aligned}
\EE\|\hy_k-\hy_{k-1}\|
&\le \alpha\,\EE\|\nabla_yG(x_k,y_k^0;\pi_k)\|  \\
&\le \alpha\left(L(\EE\|y_k^0-y_k^*\|^2)^{1/2}+\sigma\right)
\le \alpha\Gamma_y .
\end{aligned}
\]
Combining the previous displays proves~\eqref{eq:evv_corrected}. The stated corollary follows by the additional bound
$\EE\|x_k-x_{k-1}\|\le(\EE\|x_k-x_{k-1}\|^2)^{1/2}$.
}
\end{proof}
}
\section{Proof of Lemma \ref{lemma:tvk-tkvstar}}
\begin{proof}
{
By the corrected definition,
\[
\tv_k^*=[\nabla_y^2g(x_k,\hy_k)]^{-1}\nabla_yf(x_k,\hy_k).
\]
Using the same inverse perturbation calculation as in Lemma~\ref{lem:err-ehv-vstar}, for any two points $(x,y)$ and $(x',y')$,
\[
\|\tv(x,y)-\tv(x',y')\|
\le C_0\bigl(\|x-x'\|+\|y-y'\|\bigr).
\]
Applying this with $(x,y)=(x_k,\hy_k)$ and $(x',y')=(x_{k-1},\hy_{k-1})$ gives
\[
{\|\tv_k^*-\tv_{k-1}^*\|}
\le C_0\|x_k-x_{k-1}\|+C_0\|\hy_k-\hy_{k-1}\|.
\]
{Taking total expectation gives~\eqref{eq:tvk-tvk1}.}
{Squaring the same pathwise inequality, using $(a+b)^2\le 2a^2+2b^2$, and then taking total expectation gives~\eqref{eq:tvk-tvk1-square}.}
}
\end{proof}
\section{Proof of Lemma \ref{lem:hessian-tracking}}
\begin{proof}
{
Condition on the state before the adjoint update and write
\[
A=\nabla_y^2g(x_k,\hy_k),\quad
A_k=\nabla_y^2G(x_k,\hy_k;\zeta'_k),\quad
b=\nabla_yf(x_k,\hy_k),\quad
b_k=\nabla_yF(x_k,\hy_k;{\xi_k^v}).
\]
Then $\tv_k^*=A^{-1}b$ and
\[
\hv_k-\tv_k^*
=(I-\eta A)(v_k^0-\tv_k^*)-\eta\{(A_k-A)v_k^0-(b_k-b)\}.
\]
The noise term has conditional mean zero and, by Assumption~\ref{assum-5} and Lemma~\ref{lem:vk-v0-mu},
\[
{\EE_k^y}\|(A_k-A)v_k^0-(b_k-b)\|^2
\le 2\sigma^2\|v_k^0\|^2+2\sigma^2
{\le 2\sigma^2\left(1+V_0^2\right).}
\]
Since $\eta\le 1/(2L)$ and $A\succeq \mu I$,
\[
{\EE_k^y}\|\hv_k-\tv_k^*\|^2
\le (1-\mu\eta)\|v_k^0-\tv_k^*\|^2
+{4\eta^2\sigma^2\left(1+V_0^2\right).}
\]
For $k\ge1$, $v_k^0=\hv_{k-1}$. Young's inequality gives
\[
\|v_k^0-\tv_k^*\|^2
\le \left(1+\frac{\mu\eta}{2}\right)\|\hv_{k-1}-\tv_{k-1}^*\|^2
+\frac{3}{\mu\eta}\|\tv_k^*-\tv_{k-1}^*\|^2.
\]
{Taking total expectation and using the squared target-drift estimate~\eqref{eq:tvk-tvk1-square} gives the two drift terms in~\eqref{eq:htv-revised}.}
}
\end{proof}
\newpage
\section{Proof of Lemma \ref{lem:ssaid-convergence}}
\begin{proof}
{
Define
\[
{T_k:=\left(C_2\left(\EE\|\hy_k-y_k^*\|^2\right)^{1/2}
+L\left(\EE\|\hv_k-\tv_k^*\|^2\right)^{1/2}\right)^2 .}
\]
{By the squared lower-tracking recursion~\eqref{eq:yy_revised}, the Hessian-tracking recursion~\eqref{eq:htv-revised}, and the displacement bound~\eqref{eq:xxx}, there is a constant $C_1$ depending only on the smoothness constants such that}
\[
T_k
\le \left(1-\frac{\mu\alpha}{4}\right)T_{k-1}
+{2D_\alpha\beta^2\EE\|\nabla\Phi(x_{k-1})\|^2}
+D_{\mathrm{noise}},
\]
where
\[
D_\alpha:=\frac{4L^2C_2^2}{\alpha\mu^3}+4L^2C_0^2
\]
and
\[
{
D_{\mathrm{noise}}
:=
64\left(2\beta C_1 C_3
+\alpha(C_2\sigma+C_0\Gamma_y){+\eta L\sigma\sqrt{1+V_0^2}}\right)^2 .}
\]
{More explicitly, after substituting the displacement bound~\eqref{eq:xxx}, the only terms that multiply the previous tracking metric $T_{k-1}$ have coefficients of order $\beta^2C_1^2/(\mu\alpha)$. The choice $\beta\le c_\beta\kappa^{-4}\alpha$, with $c_\beta$ sufficiently small, makes their total contribution at most $(\mu\alpha/8)T_{k-1}$, changing the preliminary contraction $1-\mu\alpha/4$ into $1-\mu\alpha/8$. The lower stochastic term contributes $\alpha C_2\sigma$, the lower warm-start drift contributes $\alpha C_0\Gamma_y$, and the Hessian-vector stochastic term in~\eqref{eq:htv-revised} contributes $\eta L\sigma\sqrt{1+V_0^2}$ to the RMS residual. Unrolling the recursion proves~\eqref{eq:cc_corrected}; the residual term is $D_{\mathrm{noise}}\sum_t(1-\mu\alpha/8)^{k-1-t}$, without an additional outside factor $(\mu\alpha)^{-1}$. Jensen's inequality then gives the same upper bound with $\bar v_k$ in place of $\hv_k$ whenever the conditional hypergradient bias is used.}
}
\end{proof}
\section{Proof of Lemma \ref{lem:hypergrad-bias}}\label{proof:hypergrad-bias}
\begin{proof}
{
Let
\[
\bar\nabla\Phi_k:=\nabla_x f(x_k,\hy_k)-\nabla_{xy}^2g(x_k,\hy_k)\tv_k^* .
\]
By the same Lipschitz and inverse-perturbation argument used in Lemma~\ref{lem:err-ehv-vstar},
\[
\|\bar\nabla\Phi_k-\nabla\Phi(x_k)\|
\le C_2\|\hy_k-y_k^*\|.
\]
Moreover,
\[
{\EE_k^y[\widehat\nabla\Phi(x_k)]-\bar\nabla\Phi_k
=-\nabla_{xy}^2g(x_k,\hy_k)\bigl(\bar v_k-\tv_k^*\bigr),}
\]
{where the equality uses the independent hypergradient sample $\xi_k^x$ and the independent mixed-Hessian sample $\zeta_k$, so that the $\nabla_xF$ term averages to $\nabla_xf$ and $\zeta_k$ is independent of $\hv_k$ conditional on $\mathcal F_k^y$.}
so $\|\nabla_{xy}^2g\|\le L$, followed by Cauchy--Schwarz, gives~\eqref{eq:pp-rev}.

For the second-moment bound, decompose
\[
{
\begin{aligned}
\widehat\nabla\Phi(x_k)-\nabla\Phi(x_k)
={}&\bigl(\nabla_xF(x_k,\hy_k;{\xi_k^x})-\nabla_xf(x_k,\hy_k)\bigr)
-\bigl(\nabla_{xy}^2G-\nabla_{xy}^2g\bigr)\hv_k\\
&-\nabla_{xy}^2g(x_k,\hy_k)(\hv_k-\tv_k^*)
+\bigl(\bar\nabla\Phi_k-\nabla\Phi(x_k)\bigr),
\end{aligned}}
\]
where all derivatives are evaluated at $(x_k,\hy_k)$ unless shown otherwise. The deterministic approximation part is bounded by $C_2\|\hy_k-y_k^*\|$. For the stochastic part, Assumption~\ref{assum-5} gives
\[
\EE\|\nabla_xF(x_k,\hy_k;{\xi_k^x})-\nabla_xf(x_k,\hy_k)\|^2\le\sigma^2
\]
and
\[
\EE\|(\nabla_{xy}^2G(x_k,\hy_k;\zeta_k)-\nabla_{xy}^2g(x_k,\hy_k))\hv_k\|^2
\le \sigma^2\EE\|\hv_k\|^2.
\]
{Combining these estimates with $(a+b+c+d)^2\le4(a^2+b^2+c^2+d^2)$ proves~\eqref{eq:ppp-rev}.}
}
\end{proof}
\section{Proof of Lemma \ref{lem:hypergradient-bias-exp}}
\begin{proof}
{
{The conditional-bias estimate follows by squaring~\eqref{eq:pp-rev}, using Jensen's inequality $\EE\|\bar v_\ell-\tv_\ell^*\|^2\le \EE\|\hv_\ell-\tv_\ell^*\|^2$, summing over $\ell=0,\ldots,k-1$, and applying the RMS combined tracking bound~\eqref{eq:cc_corrected}. Lemma~\ref{lem:seq} converts each weighted convolution into one factor of order $(\mu\alpha)^{-1}$; the residual convolution therefore contributes $k(\mu\alpha)^{-1}$ times the bracketed squared residual in~\eqref{eq:pp_s}, not a factor $(\mu\alpha)^{-2}$ without $k$.}

{For the squared estimator error, we use the direct second-moment estimate~\eqref{eq:ppp-rev}. Summing~\eqref{eq:ppp-rev} over $\ell=0,\ldots,k-1$ and applying the squared lower-tracking and Hessian-tracking recursions~\eqref{eq:yy_revised} and~\eqref{eq:htv-revised} gives the same convolution structure as~\eqref{eq:pp_s}, with doubled constants because the estimator MSE contains $\EE\|\hv_\ell-\tv_\ell^*\|^2$ rather than only $\EE\|\bar v_\ell-\tv_\ell^*\|^2$. The uniform bound $\EE\|\hv_\ell\|^2\le V_0^2$ contributes the additive variance term $kB_\sigma$. This proves~\eqref{eq:pp_ss}; the corrected gradient-sum index is $\ell$, not the undefined symbol $t$.}
}
\end{proof}

\section{Proof of Theorem \ref{theorem:1}}
\begin{proof}
{
{Use the stagewise conditional expectation from Definition~2. Since $x_k$ is $\mathcal{F}_k^0$-measurable, the $L_\Phi$-smoothness of $\Phi$ and the update $x_{k+1}=x_k-\beta\widehat\nabla\Phi(x_k)$ give}
\begin{align*}
{\EE_k^0[\Phi(x_{k+1})]}
&\le \Phi(x_k)-\beta\langle\nabla\Phi(x_k),{\EE_k^0[\widehat\nabla\Phi(x_k)]}\rangle
+\frac{L_\Phi\beta^2}{2}{\EE_k^0}\|\widehat\nabla\Phi(x_k)\|^2\\
&\le \Phi(x_k)-\frac{\beta}{2}\|\nabla\Phi(x_k)\|^2
+\frac{\beta}{2}{\|\EE_k^0[\widehat\nabla\Phi(x_k)]-\nabla\Phi(x_k)\|^2}\\
&\quad +L_\Phi\beta^2\|\nabla\Phi(x_k)\|^2
+L_\Phi\beta^2{\EE_k^0}\|\widehat\nabla\Phi(x_k)-\nabla\Phi(x_k)\|^2 .
\end{align*}
{Because $\mathcal{F}_k^0\subset\mathcal{F}_k^y$ and $\nabla\Phi(x_k)$ is $\mathcal{F}_k^0$-measurable, Jensen's inequality gives
\[
\EE\|\EE_k^0[\widehat\nabla\Phi(x_k)]-\nabla\Phi(x_k)\|^2
\le
\EE\|\EE_k^y[\widehat\nabla\Phi(x_k)]-\nabla\Phi(x_k)\|^2 .
\]}
Since $\beta\le1/(8L_\Phi)$, taking total expectation and summing from $\ell=0$ to $K-1$ yields
\begin{align*}
\frac{\beta}{4}\sum_{\ell=0}^{K-1}\EE\|\nabla\Phi(x_\ell)\|^2
\le{}&
\Phi(x_0)-\Phi_*\\
&+{\frac{\beta}{2}\sum_{\ell=0}^{K-1}
\EE\|\EE_\ell^y[\widehat\nabla\Phi(x_\ell)]-\nabla\Phi(x_\ell)\|^2}\\
&+L_\Phi\beta^2\sum_{\ell=0}^{K-1}
\EE\|\widehat\nabla\Phi(x_\ell)-\nabla\Phi(x_\ell)\|^2 .
\end{align*}
{Let
\[
\mathcal T_0:=
\left(C_2 \left(\EE\norm{\hy_0-y_0^*}^2\right)^{1/2}
+L\left(\EE\norm{\hv_0-\tv_0^*}^2\right)^{1/2}\right)^2 .
\]}
{Substituting the cumulative bounds~\eqref{eq:pp_s} and~\eqref{eq:pp_ss}, and choosing $c_\beta$ sufficiently small so that the resulting gradient-sum terms are absorbed into the left-hand side, gives}
\[
\frac{1}{K}\sum_{\ell=0}^{K-1}\EE\|\nabla\Phi(x_\ell)\|^2
\le
\mathcal{O}\!\left(
\frac{\Phi(x_0)-\Phi_*}{\beta K}
+{\frac{\mathcal T_0}{\mu\alpha K}}
+{\frac{\alpha}{\mu}}
+{\frac{\eta^2}{\mu\alpha}}
+{\frac{\beta^2}{\mu\alpha}}
+L_\Phi\beta B_\sigma
\right).
\]
{The term $\mathcal T_0/(\mu\alpha K)$ is the contribution of the initial tracking errors in~\eqref{eq:pp_s}--\eqref{eq:pp_ss}.} {The term $\beta^2/(\mu\alpha)$ comes from the $\beta$-dependent part of the bracketed residual in~\eqref{eq:pp_s}, the $\alpha/\mu$ term comes from the lower-level stochastic and warm-start residuals, and $\eta^2/(\mu\alpha)$ comes from the Hessian-vector stochastic residual in~\eqref{eq:htv-revised}.}
{With $\alpha=\eta=1/(2L\sqrt K)$ and $\beta=c_\beta\kappa^{-4}\alpha$, the additional Hessian-vector residual is also $O(\alpha/\mu)$.} {Moreover, since $\mathcal T_0=O(\kappa^4)$ for fixed initialization radii under the same convention, $\mathcal T_0/(\mu\alpha K)=O(\kappa^5K^{-1/2})$.} {Under the convention used in~\eqref{eq:constant}, $C_0=O(\kappa^2)$, $C_1=O(\kappa^3)$, $C_2=O(\kappa^2)$, $C_3=O(\kappa)$, and $L_\Phi=O(\kappa^3)$; substituting these scalings into the preceding display gives the advertised upper bound $\mathcal{O}(\kappa^7K^{-1/2})$.} Therefore
\[
\frac{1}{K}\sum_{\ell=0}^{K-1}\EE\|\nabla\Phi(x_\ell)\|^2
=\mathcal{O}(\kappa^7K^{-1/2}).
\]
{Let $\tau$ be sampled uniformly from $\{0,\ldots,K-1\}$ independently of the algorithm. Then}
\[
{\EE\|\nabla\Phi(x_\tau)\|^2
=\frac{1}{K}\sum_{\ell=0}^{K-1}\EE\|\nabla\Phi(x_\ell)\|^2.}
\]
{Thus an expected $\epsilon$-stationary point is obtained by taking $K=\mathcal{O}(\kappa^{14}\epsilon^{-2})$. Each iteration uses a constant number of gradient, Jacobian-vector, and Hessian-vector product oracle calls, so the oracle complexity is $\mathcal{O}(\kappa^{14}\epsilon^{-2})$.}
}
\end{proof}

\section{Discussions of Related Work}
\renewcommand{\theHtable}{relatedwork.\arabic{table}}
\begin{table}[!t]
    \centering
    \caption{Comparison of oracle complexity to find an $\epsilon$-stationary point for the stochastic bilevel problem \eqref{eq:bo-prob} under the standard nonconvex-strongly-convex setting. For TTSA\citep{hong2023two}, although it utilizes two-timescale updates, the high truncation level of the Neumann series required for Hessian-inverse approximation prevents it from being a truly low-cost single-loop algorithm in practice.}
    \label{tab:intro-compare}
    \begin{tabular}{lll}
        \toprule
        \textbf{Type} & \textbf{Algorithm} & \textbf{Oracle Complexity} \\
        \midrule
        Multi-loop & BSA \citep{ghadimi2018approximation} & $\cO(\kappa^9 \epsilon^{-3})$ \\
        Two-timescale& TTSA \citep{hong2023two} & $\cO(\kappa^{16} \epsilon^{-5/2})$ \\
        Multi-loop & stocBiO \citep{ji2021bilevel} & $\cO(\kappa^9 \epsilon^{-2})$ \\
        Multi-loop & AmIGO \citep{arbel2021amortized} & $\cO(\kappa^9 \epsilon^{-2})$ \\
        \rowcolor[gray]{0.9}
        Single-loop & SSAID & $\cO(\kappa^{14} \epsilon^{-2})$ \\
        \bottomrule
    \end{tabular}
\end{table}
As illustrated in Table \ref{tab:intro-compare}, multi-loop methods like stocBiO achieve a convergence rate of $\mathcal{O}(\epsilon^{-2})$ by employing multiple iterations to refine their estimates. Conversely, the design of TTSA significantly complicates the convergence analysis, resulting in a degraded complexity of $\mathcal{O}(\epsilon^{-2.5})$. Moreover, this work firstly provide the convergence rate of $\cO(\kappa^{14}\epsilon^{-2})$ for the popular single-loop algorithm SSAID.
\end{document}